\documentclass{article}
\usepackage[english]{babel}
\usepackage{amsmath,amssymb}
\usepackage{natbib}
\usepackage{amsthm}
\usepackage{hyperref}

\newcommand{\Omicron}{\mathrm{O}}
\newcommand{\barsuchthat}{|}
\newcommand{\tmaffiliation}[1]{\\ #1}
\newcommand{\tmem}[1]{{\em #1\/}}
\newcommand{\tmemail}[1]{\\ \textit{Email:} \texttt{#1}}
\newcommand{\tmop}[1]{\ensuremath{\operatorname{#1}}}
\newcommand{\tmscript}[1]{\text{\scriptsize{$#1$}}}
\newcommand{\tmtextbf}[1]{{\bfseries{#1}}}
\newcommand{\tmtextit}[1]{{\itshape{#1}}}
\newtheorem{corollary}{Corollary}
\newtheorem{proposition}{Proposition}
\newtheorem{remark}{Remark}
\newtheorem{theorem}{Theorem}

\begin{document}

\title{On the validity of kernel approximations for orthogonally-initialized neural networks}

\author{
  James Martens
  \tmaffiliation{DeepMind, London, UK}
  \tmemail{james.martens@gmail.com}
}

\maketitle

\begin{abstract}
  In this note we extend kernel function approximation results for neural networks with Gaussian-distributed weights to single-layer networks initialized using Haar-distributed random orthogonal matrices (with possible rescaling). This is accomplished using recent results from random matrix theory.
\end{abstract}

\section{Initial definitions and statement of problem}

The main object of our analysis will be a \tmtextbf{\tmtextit{fully-connected ``combined layer''}} $f$, which is defined by
\[ f (z) = \phi (Wz), \]
where $W \in \mathbb{R}^{m \times k}$ is matrix of \tmtextit{\tmtextbf{weights}}, $z \in \mathbb{R}^k$ is a input vector, and $\phi$ is a scalar \tmtextbf{\tmtextit{activation function}} which is applied element-wise for vector inputs. We call this a ``combined layer'', as it combines both a fully-connected layer and a nonlinear layer.

We are interested in characterizing the high probability initialization-time behavior of $f$'s \tmtextbf{\tmtextit{kernel function}}
\[ \kappa_f (z, z') \equiv \frac{1}{m} f (z)^{\top} f (z') \]
given certain initialization schemes for $W$. In particular, we wish to show that $\kappa_f (z, z')$, evaluated with $f$'s initial weights, converges in probability as $m \rightarrow \infty$ to $f$'s \tmtextbf{\tmtextit{approximate kernel function}}, which is defined by
\[ \widetilde{\kappa}_f (\Sigma_{z, z'}) =\mathbb{E}_{u \sim N (0, \Sigma_{z, z'})} [\phi (u_1) \phi (u_2)], \]
where
\[ \Sigma_{z, z'} \equiv \frac{1}{k}  \left[\begin{array}{cc}
     z^{\top} z & z^{\top} z'\\
     z^{\prime \top} z & z^{\prime \top} z'
   \end{array}\right] \in \mathbb{R}^{2 \times 2} . \]
(Note that this definition still makes sense when $z = z'$ if we treat degenerate multivariate Gaussian distributions in the standard way.)

We are also interested in characterizing the rate of this convergence as a function of $m$, as this gives us some idea of how wide a layer needs to be in practice for its kernel function to be well approximated.

\section{Prior result for Gaussian weights}

When the elements of $W$ are drawn iid from $N (0, 1 / k)$, this is known as a \tmtextbf{\tmtextit{Gaussian fan-in initialization}} \citep{lecun1998efficient}. Given such an initialization, we have the following result due to \citet{daniely2016toward}, which we have adapted here for single layer networks.

\begin{theorem}[Adapted from Theorem 2 of \citet{daniely2016toward}]
  \label{thm:error-bound-daniely}Suppose $W$ is initialized according to a Gaussian fan-in initialization, and that $\phi$ is a twice continuously differentiable function satisfying $\mathbb{E}_{x \sim N (0, 1)} [\phi (x)^2] = 1$ and $\| \phi \|_{\infty}, \| \phi' \|_{\infty}, \| \phi'' \|_{\infty} \leqslant C$ for some $C$ (with $\| \cdot \|_{\infty}$ denoting the supremal value). Suppose further that
  \[ m \geqslant \frac{4 C^4 \log (8 / \delta)}{\epsilon^2} . \]
  Then at initialization time, for all input vectors $z$ and $z'$ to $f$ satisfying $\| z \|^2 = \| z' \|^2 = k$, we have
  \[ | \kappa_f (z, z') - \widetilde{\kappa}_f (\Sigma_{z, z'}) | \leqslant \epsilon \]
  with probability at least $1 - \delta$.
\end{theorem}

\begin{remark}
  Because $1 = \sqrt{\mathbb{E}_{u \sim N (0, 1)} [\phi (u)^2]} \leqslant \sqrt{\mathbb{E}_{u \sim N (0, 1)} [\| \phi \|_{\infty}^2]} = \| \phi \|_{\infty} \:$, it thus follows that $C \geqslant 1$ in the above theorem.
\end{remark}

\begin{remark}
  In addition to being restricted to the single layer case, this theorem statement differs from the one in \citet{daniely2016toward} by explicitly assuming that the activation function $\phi$ satisfies $\mathbb{E}_{x \sim N (0, 1)} [\phi (x)^2] = 1$, or is in other words ``normalized". As far as we can tell, this assumption is implicit in the definitions made by \citet{daniely2016toward}. 
\end{remark}

\begin{remark}
  Another assumption implicit in the definitions of \citet{daniely2016toward} is that the activation function of the last layer is the identity, which would seem to render their result trivial for one layer networks (and the above adaptation unjustified). However, as far as we can tell, this assumption isn't actually used anywhere in their proofs.
\end{remark}

\begin{remark}
  The hypothesis that $ \| \phi'' \|_{\infty} \leqslant C$ can likely be removed in the single layer case since it appears to be used by \citet{daniely2016toward} only to bound the expansion of the approximation error across multiple layers.
\end{remark}

\section{Main result}

The goal of this note is to prove an analogous result to Theorem \ref{thm:error-bound-daniely} for the case where $W$ is initialized according to what we will call a \tmtextbf{\tmtextit{``scaled-corrected uniform orthogonal (SUO) distribution"}}. This is a distribution of scaled orthogonal matrices (or submatrices of these, when $m \neq k$) which is sometimes used in neural network training, and which has appeared in various theoretical works \citep[e.g.][]{sokol2018information, hu2020provable}. While approximate kernel function analysis (or signal propagation/mean-field analysis, which makes equivalent predictions from non-rigorous foundations) has been used in the context of orthogonally-initialized neural networks before \citep[e.g.][]{saxe2013exact,xiao2018dynamical}, to the best of our knowledge this practice has never before been rigorously justified.

When $m \leqslant k$, samples from the SUO distribution can be generated as $(XX^{\top})^{- 1 / 2} X$, where $X$ is a $m \times k$ matrix with entries sampled iid from $N (0, 1)$. When $m > k$, we may apply the same procedure but with $k$ and $m$ reversed, and then transpose the result. To be consistent with the scaling characteristics of the Gaussian fan-in initialization, we multiply the resulting matrix by $\max \left( \sqrt{m / k}, 1 \right)$, which will have an effect only when $k \leq m$. Note that without this multiplicative rescaling, the distribution corresponds to the well-known \tmtextit{\tmtextbf{Haar measure}}.

Our main result is the following theorem.

\begin{theorem}
  \label{thm:main-result}Suppose $W$ is initialized according to an SUO distribution, and that $\phi$ satisfies $\| \phi \|_{\infty}, \| \phi' \|_{\infty} \leqslant C$ for some $C$ (with $\| \cdot \|_{\infty}$ denoting the supremal value). Denote $n = \max (k, m)$, and suppose that for $\delta, \epsilon \geqslant 0$ we have
  \[ \frac{m^{5 / 2}}{(n + 1)^2} \geqslant \log (2 / \delta) \text{\qquad and\qquad} \frac{n - 1}{m^{3 / 4}} \geqslant \frac{8 \sqrt{2} C^2}{\epsilon} . \]

  Then, at initialization time, for all pairs of vectors $z, z' \in \mathbb{R}^k$ satisfying $\| z \|^2 = \| z' \|^2 = k$, we have that
  \[ | \kappa_f (z, z') - \widetilde{\kappa}_f (\Sigma_{z, z'}) | \leqslant \epsilon \]
  with probability at least $1 - \delta$.
\end{theorem}

\begin{remark}
  The conditions on $k$, $m$, and $n \equiv \max (k, m)$ in the theorem statement will be satisfied as long as $n$ is sufficiently large and $k$ is not too much larger than $m$. In the case where $m \geqslant k$, the LHS's of these bounds simplifies to approximately $m^{1 / 2}$ and $m^{1 / 4}$, respectively. It thus follows that $\kappa_f (z, z')$ converges in probability to $\widetilde{\kappa}_f (\Sigma_{z, z'})$ as the output dimension $m$ goes to $\infty$.
\end{remark}

\begin{remark}
  In the case where $m \geqslant k$, the conditions imply that
  \[ n \gtrsim \frac{128 C^4 \log (2 / \delta)^2}{\epsilon^2}, \]
  which is similar to the condition from Theorem \ref{thm:error-bound-daniely} for Gaussian fan-in initializations.
\end{remark}

\begin{remark}
  The result may be generalized to the case where $\| z \|^2 = \| z' \|^2 = \alpha^2 k$ for some $\alpha$, in which case $C$ is replaced by $\alpha C$. To see this, observe that $f (\alpha z) = \psi (Wz)$, where we have defined $\psi (x) = \phi (\alpha x)$, and that $\| \psi \|_{\infty}, \| \psi' \|_{\infty} \leqslant \alpha C$.
\end{remark}

\section{Proof of main result}

To prove Theorem \ref{thm:main-result}, we will use some fairly recent results from the probability/statistics literature. The first of these results, due to \citet{chatterjee2007multivariate}, is stated below. It makes use of the \tmtextbf{\tmtextit{1st Wasserstein distance}}, which for two random variables $X$ and $Y$of the same dimension $d$ is defined by
\[ D_{\tmop{Wass}} (X, Y) \equiv \sup \{ | \mathbb{E} [f (X)] -\mathbb{E} [f (Y)] | \barsuchthat f : \mathbb{R}^d \rightarrow \mathbb{R}, L_f = 1 \}, \]
where $L_f$ denotes the Lipschitz constant of $f$.

\begin{theorem}[Adapted from Theorem 11 of {\citet{chatterjee2007multivariate}}]
  \label{thm:haar-wass-bound}Let $B_1, B_2, \ldots, B_d \in \mathbb{R}^{n \times n}$ be linearly independent matrices such that $\tmop{tr} (B_i B_i^{\top}) = n$. Let $M$ be random $n \times n$ orthogonal matrix distributed according to the Haar measure, and define the $d$-dimensional random vector
  \[ X = (\tmop{tr} (B_1 M), \tmop{tr} (B_2 M), \ldots, \tmop{tr} (B_d M)) . \]
  Then for $n \geqslant 2$ we have
  \[ D_{\tmop{Wass}} (X, Y) \leqslant \frac{d \sqrt{2 \| A \|}}{n - 1}, \]
  where $A \in \mathbb{R}^{d \times d}$ is defined by $[A]_{i, j} = \frac{1}{n} \tmop{tr} (B_i B_j^{\top})$, $Y \sim N (0, A)$, and $\| A \|$ is the operator norm of A.
\end{theorem}

From this theorem we will prove the following corollary, which is specialized to our particular situation. \ It shows that multiplication of constant vectors by an SUO distributed matrix produces an output which is approximately Gaussian distributed, according to the 1st Wasserstein distance.

Note that throughout the rest of this section we will assume most of the hypotheses of Theorem \ref{thm:main-result}. Namely that $W$ is SUO distributed, $\phi$ satisfies $\| \phi \|_{\infty}, \| \phi' \|_{\infty} \leqslant C$ for some $C$, and that $z, z' \in \mathbb{R}^k$ are vectors satisfying $\| z \|^2 = \| z' \|^2 = k$.

\begin{corollary}
  \label{cor:ortho-vs-gauss}Suppose that $z \neq z'$, and that $y, y'$ are $m$-dimensional vectors distributed according to $\left[\begin{array}{c}
    y\\
    y'
  \end{array}\right] \sim N \left( 0, \left[\begin{array}{cc}
    1 & c\\
    c & 1
  \end{array}\right] \otimes I_m \right)$, where $c = \frac{1}{k} z^{\top} z' = z^{\top} z' / (\| z \|  \| z' \|) \neq 1$. Then for $n \equiv \max (k, m) \geqslant 2$ we have
  \[ D_{\tmop{Wass}} ((Wz, Wz'), (y, y')) \leqslant \frac{4 m}{n - 1} . \]
\end{corollary}

\begin{proof}
  First, observe that the ``multiplier'' $\max \left( \sqrt{m / k}, 1 \right)$ used in the construction of the SUO distribution is equal to $\sqrt{n / k}$.
  
  Because $W$ is SUO distributed, it follows that $W / \sqrt{n / k}$ is a random $m \times k$ orthogonal matrix distributed according to the Haar measure, and can thus be viewed as the top left submatrix of a larger $n \times n$ random orthogonal matrix $M$, also distributed according to the Haar measure (on a different space). This is a known property of the Haar measure for random orthogonal matrices \citep[Chapter 7]{eaton1989group}.
  
  For $i = 1, 2, \ldots, m$, let $B_i \in \mathbb{R}^{n \times n}$ be matrices formed by taking the zero matrix and replacing the first $k$ entries of the $i$-th column by the vector $(\sqrt{n / k}) z \in \mathbb{R}^k$. And for $i = m + 1, m + 2, \ldots, 2 m$ similarly let $B_i$ be matrices formed by taking the zero matrix and replacing the first $k$ entries of the $(i - m)$-th column by the vector $(\sqrt{n / k}) z'$. It is straightforward to verify from these definitions that $(\tmop{tr} (B_1 M), \tmop{tr} (B_2 M), \ldots, \tmop{tr} (B_{2 m} M)) = (Wz, Wz')$.
  
  We also observe that since $z \neq z'$, these $B_i$'s are linearly independent. Moreover, $\tmop{tr} (B_i B_i^{\top}) = \left( \sqrt{n / k} \right)^2 z^{\top} z = (n / k) k = n$ for $i = 1, 2, \ldots, m$ and similarly for $i = m + 1, m + 2, \ldots, 2 m$. Also, for $i = 1, \ldots, m$ we have that $\tmop{tr} (B_i B_{i+m}^{\top}) = \tmop{tr} (B_{i+m} B_i^{\top}) = \left( \sqrt{n / k} \right)^2 z^{\top} z' = (n / k) z^{\top} z' = nc$. Meanwhile, for all other combinations of $i$ and $j$ not covered by these cases we have $\tmop{tr} (B_i B_{j}^{\top}) = \tmop{tr} (B_{j} B_i^{\top}) = 0$.
  
  Defining $A \in \mathbb{R}^{d \times d}$ by $[A]_{i, j} = \frac{1}{n} \tmop{tr} (B_i B_j^{\top})$, we thus have
  \[ A = \frac{1}{n}  \left[\begin{array}{cc}
       n & nc\\
       nc & n
     \end{array}\right] \otimes I_m = \left[\begin{array}{cc}
       1 & c\\
       c & 1
     \end{array}\right] \otimes I_m . \]
  Since $- 1 \leqslant c \leqslant 1$, it follows that $\left\| \left[\begin{array}{cc}
    1 & c\\
    c & 1
  \end{array}\right] \right\| \leqslant 2$, and thus $\| A \| = \left\| \left[\begin{array}{cc}
    1 & c\\
    c & 1
  \end{array}\right] \right\|  \| I_m \| \leqslant 2$.
  
  Applying Theorem \ref{thm:haar-wass-bound} with $d = 2 m$ we finally get that
  \[ D_{\tmop{Wass}} ((Wz, Wz'), (y, y')) \leqslant \frac{2 m \sqrt{2 \| A \|}}{n - 1} \leqslant \frac{4 m}{n - 1}, \]
  where $y$ and $y'$ are defined as in the statement.
\end{proof}

Corollary $\ref{cor:ortho-vs-gauss}$ is useful because it allows us to relate $\mathbb{E} [\kappa_f (z, z')]$ to $\widetilde{\kappa}_f (\Sigma_{z, z'})$, given certain assumptions on $f$'s activation function $\phi$. This is done in the following proposition.

\begin{proposition}
  \label{prop:mean-vs-approx}For $n \equiv \max (k, m) \geqslant 2$ we have
  \[ | \mathbb{E} [\kappa_f (z, z')] - \widetilde{\kappa}_f (\Sigma_{z, z'}) | \leqslant \frac{4 \sqrt{2}  \sqrt{m} C^2}{n - 1} . \]
\end{proposition}

\begin{proof}
  Denote
  \[ g (x, x') \equiv \frac{1}{m} \phi (x)^{\top} \phi (x') = \frac{1}{m}  \sum^m_{i = 1} \phi ([x]_i) \phi ([x']_i) . \]
  The gradient of $g$ with respect to $x$ and $x'$ is given by
  \[ \frac{1}{m}  \left[\begin{array}{c}
       \phi' (x) \odot \phi (x')\\
       \phi (x) \odot \phi' (x')
     \end{array}\right] \in \mathbb{R}^{2 m}, \]
  where $\odot$ denotes the element-wise product.
  
  As we have $\| \phi \|_{\infty}, \| \phi' \|_{\infty} \leqslant C$, it follows that each entry of this vector is upper bounded by $C^2 / m$. Its norm is thus upper bounded by
  \[  \sqrt{2 m \left( \frac{C^2}{m} \right)^2} = \frac{\sqrt{2} C^2}{\sqrt{m}} . \]
  Since this bound is universal, it is also an upper bound on the Lipschitz constant $L_g$ of $g$.
  
  Suppose that $z \neq z'$. Since $g (x, x') / L_g$ has a Lipschitz constant of 1, we have by Corollary \ref{cor:ortho-vs-gauss} and the definition of $D_{\tmop{Wass}}$ that
  \[ \left| \mathbb{E} \left[ \frac{g (Wz, Wz')}{L_g} \right] -\mathbb{E} \left[ \frac{g (y, y')}{L_g} \right] \right| \leqslant D_{\tmop{Wass}} ((Wz, Wz'), (y, y')) \leqslant \frac{4 m}{n - 1} \]
  for $\left[\begin{array}{c}
    y\\
    y'
  \end{array}\right] \sim N \left( 0, \left[\begin{array}{cc}
    1 & c\\
    c & 1
  \end{array}\right] \otimes I_m \right)$, which implies that
  \[ | \mathbb{E} [g (Wz, Wz')] -\mathbb{E} [g (y, y')] | \leqslant \frac{4 mL_g}{n - 1} \leqslant \frac{4 \sqrt{2}  \sqrt{m} C^2}{n - 1} . \]
  Observing that $\kappa_f (z, z') = g (Wz, Wz')$ by definition of $\kappa_f$, and also that
  \begin{align*}
  \mathbb{E} [g (y, y')] &= \frac{1}{m}  \sum^m_{i = 1} \mathbb{E} [\phi ([x]_i) \phi ([x']_i)] \\ 
  &= \frac{1}{m} m \mathbb{E}_{\tmscript{\left[\begin{array}{c}
       u_1\\
       u_2
     \end{array}\right] \sim N \left( 0, \left[\begin{array}{cc}
       1 & c\\
       c & 1
     \end{array}\right] \right)}} [\phi (u_1) \phi (u_2)] = \widetilde{\kappa}_f (\Sigma_{z, z'})
  \end{align*}
  (where we have used the fact that $\Sigma_{z, z'} = \left[\begin{array}{cc}
    1 & c\\
    c & 1
  \end{array}\right]$), the proposition then follows for the case $z \neq z'$.
  
  It remains to handle the $z = z'$ case. We observe that both $\kappa_f (z, z')$ and $\widetilde{\kappa}_f (\Sigma_{z, z'})$ are continuous functions of $z'$, and thus so is $|\kappa_f (z, z') - \widetilde{\kappa}_f (\Sigma_{z, z'})|$. It thus follows from the $z \neq z'$ case that
  \begin{align*}
  |\kappa_f (z, z) - \widetilde{\kappa}_f (\Sigma_{z, z})| = \lim_{\substack{z' \to z \\ z' \neq z, \|z'\|^2 = k}} |\kappa_f (z, z') - \widetilde{\kappa}_f (\Sigma_{z, z'})| \leqslant \frac{4 \sqrt{2}  \sqrt{m} C^2}{n - 1} .
  \end{align*}

\end{proof}

Having characterized the speed at which $\mathbb{E} [\kappa_f (z, z')]$ converges to $\widetilde{\kappa}_f (\Sigma_{z, z'})$ as $n = \max (k, m)$ grows, it remains to characterize the rate at which $\kappa_f (z, z')$ concentrates around its expectation. For Gaussian distributed weights this is easy, since the input to each unit becomes iid, allowing one to use well-known and powerful concentration bounds for sums of iid variables, such as Hoeffding's inequality. (This is the approach taken by \citet{daniely2016toward}.) However, one does not have this kind of independence for SUO distributed weights, and so we must use a more specialized approach.

Following the strategy outlined by \citet{meckes2019random}, we will use a concentration inequality that applies to distributions on metric spaces that satisfy so-called log-Sobolev inequalities (LSIs). Note that the Haar measure on the space $\Omicron (n)$ of $n \times n$ orthogonal matrices, with the Hilbert-Schmidt distance metric (given by $\| \tmop{vec} (M_1) - \tmop{vec} (M_2) \|$), does not satisfy an LSI due to $\Omicron (n)$ being disconnected. Fortunately, the same measure restricted to either of $\Omicron (n)$'s two connected components (and then renormalized) does satisfy such an inequality. These components are $\tmop{SO} (n)$ and $\tmop{SO}^- (n)$, which correspond to orthogonal matrices with determinants $1$ and $- 1$ respectively (which are the only two possible values). As shown by \citet{meckes2019random}, the constant of the LSI inequality is $4 / (n - 2)$ for $\tmop{SO} (n)$, which leads to the theorem stated below.

Before we state the theorem we will quickly define some additional notation. Given a set $T$ of orthogonal matrices, we denote by $H (T)$ the distribution on $T$ given by the Haar measure. We will also extend the notation ``$\Omicron (n)$'' to include sets $\Omicron (m, k)$ of non-square $m \times k$ orthogonal matrices.

\begin{theorem}[Adapted from Theorem 5.5 of \citet{meckes2019random}, as applied to $\tmop{SO} (n)$]
  \label{thm:meckes-concentration}Suppose that $M \sim H (\tmop{SO} (n))$ and that $h : \tmop{SO} (n) \rightarrow \mathbb{R}$ is a Lipschitz-continuous function with constant $L_h$ (in the Hilbert-Schmidt metric space) such that $\mathbb{E} [| h (M) |] < \infty$. Then for every $r \geqslant 0$ we have
  \[ P (| h (M) -\mathbb{E} [h (M)] | \geqslant r) \: \leqslant \: 2 \exp \left( - \frac{r^2  (n - 2)}{8 L_h^2} \right) . \]
\end{theorem}

We now have all of the tools required to prove our main result.

\begin{proof}[Proof of Theorem \ref{thm:main-result}]
  Let $R$ be a $k \times k$ orthogonal matrix that maps $z$ and $z'$ into the subspace $S$ spanned by the unit vectors $e_1$ and $e_2$ (or just $e_1$ if $z=z'$). We will begin the proof by showing that $(Wz, Wz')$ is distributed as $(Vv, Vv')$, where $(v, v') = (Rz, Rz')$, and $V / \sqrt{n / k}$ is the top left submatrix of a matrix $M \sim H (\tmop{SO} (n))$.
  
  Because $W$ is SUO distributed, $W / \sqrt{n / k}$ has distribution $H (O (m, k))$, which is invariant to right multiplication by orthogonal matrices such as $R$ {\cite[Chapter 7]{eaton1989group}}, and thus the distribution of $W$ is also invariant. It therefore follows that $((WR) z, (WR) z') = (Wv, Wv')$ has the same distribution as $(Wz, Wz')$.
  
  Since $W / \sqrt{n / k}$ has distribution $H (O (m, k))$, it is distributed as the top-left submatrix of a matrix $Q \sim H (\Omicron (n))$ {\cite[Chapter 7]{eaton1989group}}. Because $(Wv, Wv')$ depends only on the first one or two column(s) of $W$, it thus depends only on the first one or two column(s) of $Q$. As argued below, the first $n - 1$ columns of an $n \times n$ matrix has the same distribution whether that matrix is distributed according to $H (\Omicron (n))$, $H (\tmop{SO} (n))$ or $H (\tmop{SO}^- (n))$, and so for $n \geqslant 3$ we have that $(Wv, Wv')$ is distributed as $(Vv, Vv')$, where $V / \sqrt{n / k}$ is the top-left submatrix of a matrix $M$ with distribution $H (\tmop{SO} (n))$.
  
  To see that the first $n - 1$ columns of an $n \times n$ matrix have the same distribution whether that matrix is distributed according to $H (\Omicron (n))$, $H (\tmop{SO} (n))$ or $H (\tmop{SO}^- (n))$, consider the effect of multiplying by the orthogonal matrix $T = \left[\begin{array}{cc}
    I_{n - 1} & 0\\
    0 & - 1
  \end{array}\right]$. This keeps $H (\Omicron (n))$ invariant, doesn't change the first $n - 1$ columns of the matrix, and flips membership between $\tmop{SO} (n)$ and $\tmop{SO}^- (n)$. Because the Haar measure on $\tmop{SO} (n)$ and $\tmop{SO}^- (n)$ is given by the restriction of the Haar measure on $\Omicron (n)$ (followed by renormalization), it follows that multiplication by $T$ is a measure preserving transform between $H (\tmop{SO} (n))$ and $H (\tmop{SO}^- (n))$ which doesn't change the first $n - 1$ columns.
  
  The remainder of the proof is now just an application of Proposition \ref{prop:mean-vs-approx} and Theorem \ref{thm:meckes-concentration}.
  
  Define
  \[ g (x, x') \equiv \frac{1}{m} \phi (x)^{\top} \phi (x') = \frac{1}{m}  \sum^m_{i = 1} \phi ([x]_i) \phi ([x']_i), \]
  so that $\kappa_f (z, z') = g (Wz, Wz')$. Given the above relationship between $M$ and $V$ (i.e.~that $V$ is the top-left submatrix of $\sqrt{n / k} M$), we define
  \[ h (M) \equiv g (Vv, Vv') . \]
  To bound the Lipschitz constant $L_h$ of $h$ in the Hilbert-Schmidt metric it suffices to compute the gradient $G$ of $h$ with respect to $M$, and bound its Hilbert-Schmidt norm (which is given by its standard Euclidean norm as a vector in $\mathbb{R}^{n^2}$, or equivalently by $\tmop{tr} (GG^{\top})^{1 / 2}$). As the gradient is zero for all entries of $M$ except the top-left $m \times k$ submatrix (used to form $V / \sqrt{n / k}$), we can restrict our computation to these entries, which gives
  \[ G = \frac{\sqrt{n}}{m \sqrt{k}}  [(\phi (Vv') \odot \phi' (Vv)) v^{\top} + (\phi (Vv) \odot \phi' (Vv')) v^{\prime \top}] . \]
  Observing that $\tmop{tr} (ab^{\top} cd^{\top}) = \tmop{tr} (b^{\top} cd^{\top} a) = (b^{\top} c) (d^{\top} a)$ for vectors $a$, $b$, $c$, and $d$, and defining $s_1 = (\phi (Vv') \odot \phi' (Vv)) \in \mathbb{R}^m$ and $s_2 = (\phi (Vv) \odot \phi' (Vv') \in \mathbb{R}^m)$ we have that the Hilbert-Schmidt norm of $G$ is thus equal to
  \[ \tmop{tr} (GG^{\top})^{1 / 2} = \frac{\sqrt{n}}{m \sqrt{k}}  \sqrt{\| s_1 \|^2  \| v \|^2 + \| s_2 \|^2  \| v' \|^2 + 2 (s_1^{\top} s_2)  (v^{\top} v')} . \]
  As $\phi$ and $\phi'$ are bounded by $C$ by hypothesis, we have $\| s_1 \|^2, \| s_2 \|^2, | s_1^{\top} s_2 | \leqslant mC^4$. And because $R$ is orthogonal we have $\| v \|^2 = \| Rz \|^2 = \| z \|^2 = k$ and similarly $\| v' \|^2 = k$, and also $| v^{\top} v' | = | z^{\top} z' | \leqslant k$. Combining these inequalities we get the following upper bound on the norm of $G$ (and thus on $L_h$):
  \[ \frac{\sqrt{n}}{m \sqrt{k}}  \sqrt{4 mC^4 k} \: = \: \frac{2 C^2  \sqrt{n}}{\sqrt{m}} . \]
  Since $h$ is clearly bounded (because $\phi$ is), we have that $\mathbb{E} [| h (M) |] < \infty$.
  
  We have now satisfied the hypotheses of Theorem \ref{thm:meckes-concentration}, which for $h (M) = g (Vv, Vv')$, gives us that
  \[ P (| g (Vv, Vv') -\mathbb{E} [g (Vv, Vv')] | \geqslant r) \leqslant 2 \exp \left( - \frac{r^2  (n - 2)}{8 L_h^2} \right) \leqslant 2 \exp \left( - \frac{r^2  (n - 2) m}{32 nC^4} \right) . \]
  By Proposition \ref{prop:mean-vs-approx} we also have
  \[ | \mathbb{E} [\kappa_f (z, z')] - \widetilde{\kappa}_f (\Sigma_{z, z'}) | \leqslant \frac{4 \sqrt{2}  \sqrt{m} C^2}{n - 1} . \]
  Let $\epsilon$ be such that
  \begin{equation}
    \epsilon \geqslant \frac{4 \sqrt{2} m^{3 / 4} C^2}{n - 1} + \frac{4 \sqrt{2}  \sqrt{m} C^2}{n - 1} . \label{eqn:eps-ineq}
  \end{equation}
  
  Using the above three inequalities we have
    \begin{align*}
    P & (| \kappa_f (z, z') - \widetilde{\kappa}_f (\Sigma_{z, z'}) | \geqslant \epsilon)  \\
    & = P (| g (Wz, Wz') - \widetilde{\kappa}_f (\Sigma_{z, z'}) | \geqslant \epsilon)\\
    & \leqslant P (| g (Wz, Wz') -\mathbb{E} [g (Wz, Wz')] | + | \mathbb{E} [g (Wz, Wz')] - \widetilde{\kappa}_f (\Sigma_{z, z'}) | \geqslant \epsilon)\\
    & \leqslant P \left( | g (Wz, Wz') -\mathbb{E} [g (Wz, Wz')] | + \frac{4 \sqrt{2}  \sqrt{m} C^2}{n - 1} \geqslant \epsilon \right)\\
    & = P \left( | g (Wz, Wz') -\mathbb{E} [g (Wz, Wz')] | \geqslant \epsilon - \frac{4 \sqrt{2}  \sqrt{m} C^2}{n - 1} \right)\\
    & \leqslant P \left( | g (Wz, Wz') -\mathbb{E} [g (Wz, Wz')] | \geqslant \frac{4 \sqrt{2} m^{3 / 4} C^2}{n - 1} \right)\\
    & = P \left( | g (Vv, Vv') -\mathbb{E} [g (Vv, Vv')] | \geqslant \frac{4 \sqrt{2} m^{3 / 4} C^2}{n - 1} \right)\\
    & \leqslant 2 \exp \left( - \frac{\left( \frac{4 \sqrt{2} m^{3 / 4} C^2}{n - 1} \right)^2  (n - 2) m}{32 nC^4} \right)\\
    & = 2 \exp \left( - \frac{(n - 2) m^{5 / 2}}{n (n - 1)^2} \right) .
  \end{align*}
  
  By this inequality we thus have that the condition $2 \exp \left( - \frac{(n - 2) m^{5 / 2}}{n (n - 1)^2} \right) \leqslant \delta$ would be sufficient to establish that $P (| \kappa_f (z, z') - \widetilde{\kappa}_f (\Sigma_{z, z'}) | < \epsilon) = 1 - P (| \kappa_f (z, z') - \widetilde{\kappa}_f (\Sigma_{z, z'}) | \geqslant \epsilon) \geqslant 1 - \delta$. Taking logs of both sides and simplifying this condition becomes
  \[ \frac{(n - 2) m^{5 / 2}}{n (n - 1)^2} \geqslant \log (2 / \delta) . \]
  For $n \geqslant 2$ we have $(n - 2) / (n (n - 1)^2) \geqslant 1 / (n + 1)^2$, and so for the condition to be satisfied it suffices that
  \[ \frac{m^{5 / 2}}{(n + 1)^2} \geqslant \log (2 / \delta) . \]
  It remains to derive conditions on $m$ and $n$ such that Equation \ref{eqn:eps-ineq} is satisfied. Since the first term on the RHS of Equation \ref{eqn:eps-ineq} clearly upper bounds the second term, it follows that a sufficient condition is that $\epsilon$ is greater than or equal to 2 times the first term. Taking this condition and rearranging we have
  \[ \frac{n - 1}{m^{3 / 4}} \geqslant \frac{8 \sqrt{2} C^2}{\epsilon} . \]
\end{proof}

\section{A previous solution to this problem?}

\citet[arXiv version 3]{huang2020neural} claim a result which is in many ways stronger than Theorem \ref{thm:main-result}, although without an explicit convergence rate. In particular, they show that the kernel function of a \emph{multi-layer} fully-connected network initialized with SUO distributed weights converges to the composition of the approximate kernel functions for its individual layers, as the widths of the layers all go to infinity. They then extend their analysis to verify the formula for the Neural Tangent Kernel \citep{li2018learning,jacot2018neural,du2018gradient} of such networks, which had previously only been proven for Gaussian-distributed weights.

The main technical tool that they use is a simplified version of Theorem \ref{thm:haar-wass-bound} (which they also adapt from \citet{chatterjee2007multivariate}) where $d = 1$, and the type and rate of convergence of $\tmop{tr} (B_1 M)$ is left unspecified. Using this, they argue that each scalar component of $Wz$ converges to a Gaussian distribution as $z$'s dimension $k$ goes to infinity. Then, because each component of $Wz$ is Gaussian in the limit and has zero covariance with the other components, they become independent in the limit. This property is then used to argue about the behavior of $\phi (Wz)$ for an element-wise activation function $\phi$.

Unfortunately, there appears to be several errors in this argument that are not easily repaired. Firstly, even if each variable in some set is Gaussian distributed, this does not imply that they will be {\tmem{jointly}} Gaussian distributed. See \citet{30205} for various counterexamples. Secondly, it is not mathematically rigorous to talk about a set of variables ``becoming independent'' in the limit. From the standpoint of formulas and theorems that require independence (such as the Central Limit Theorem), a set of variables are either independent or they are not. Thirdly, without being precise about the type and rate of convergence experienced by a arbitrary set of 1D projections of an orthogonal matrix, one cannot make any strong statements about the limiting behavior of a function of those projections (such as the output of a fully-connected combined layer in this case). For example, if $M$ is a random $n \times n$ orthogonal matrix distributed according to the Haar measure, then for all $n \geqslant 1$, $| M |$ takes on values 1 or $- 1$ each with 0.5 probability, while the distribution of $| M |$ for an entry-wise iid Gaussian distributed $M$ looks very different. Finally, to apply their argument inductively to deeper networks, \citet{huang2020neural} take the width of the current layer to infinity at each step, giving a ``sequential limit'' over the layers. While such limits have appeared before in the literature \citep[e.g.][]{novak2018bayesian,jacot2018neural}, it is not clear that they are mathematically meaningful, as a layer in a neural network can never actually be ``infinitely wide''. Moreover, even if one accepts sequential limits as valid (e.g.~by adopting a Gaussian-process interpretation), they can never actually describe the approximate behavior of a really wide (but finite) multi-layer neural network, which arguably defeats the main purpose of computing such limits.

\section{Future directions}

While Theorem \ref{thm:main-result} only handles single combined layers, we conjecture that it could be generalized to deeper networks following an argument along the lines of the one given by \citet{daniely2016toward} for deep Gaussian-initialized networks. We leave this to future work.

\section*{Acknowledgements}
We would like to thank Elizabeth Meckes for being very helpful in answering questions about her work on Haar-distributed matrices (which this note is largely built on). We would also like to thanks Mark Rowland for reviewing a draft of this note and providing insightful comments which helped improve it significantly.

\bibliographystyle{plainnat}
\bibliography{ortho_weights_theory}

\end{document}